\documentclass[11pt,a4paper]{article}
\usepackage[hyperref]{emnlp-ijcnlp-2019}
\usepackage{times}
\usepackage{latexsym}
\usepackage{url}

\aclfinalcopy 

\usepackage{graphicx}
\usepackage{amsmath}
\usepackage{amssymb}
\usepackage{amsthm}
\usepackage{bm}
\usepackage{algorithm}
\usepackage[noend]{algpseudocode}
\usepackage{xcolor}
\usepackage{subcaption}

\newcommand{\eps}{\varepsilon}
\DeclareMathOperator*{\argmax}{arg\,max}

\newcommand{\smallsection}[1]{{\noindent\textbf{#1.}}}

\title{Phrase Grounding by Soft-Label Chain Conditional Random Field}
\author{
Jiacheng Liu $\quad$ Julia Hockenmaier \\
University of Illinois at Urbana-Champaign, Urbana, IL, USA 61801 \\
{\tt \{jl25, juliahmr\}@illinois.edu}
}
\date{}

\begin{document}
\maketitle
\begin{abstract}
The phrase grounding task aims to ground each entity mention in a given caption of an image to a corresponding region in that image. Although there are clear dependencies between how different mentions of the same caption should be grounded, previous structured prediction methods that aim to capture such dependencies need to resort to approximate inference or non-differentiable losses. In this paper, we formulate phrase grounding as a sequence labeling task where we treat candidate regions as potential labels, and use neural chain Conditional Random Fields (CRFs) to model dependencies among regions for adjacent mentions. In contrast to standard sequence labeling tasks, the phrase grounding task is defined such that there may be multiple correct candidate regions. To address this multiplicity of gold labels, we define so-called Soft-Label Chain CRFs, and present an algorithm that enables convenient end-to-end training. Our method establishes a new state-of-the-art on phrase grounding on the Flickr30k Entities dataset. Analysis shows that our model benefits both from the entity dependencies captured by the CRF and from the soft-label training regime. Our code is available at \url{github.com/liujch1998/SoftLabelCCRF}
\end{abstract}

\section{Introduction}
\label{sec:introduction}

\begin{figure}[t]
    \centering
    \begin{subfigure}[b]{0.48 \columnwidth}
        \includegraphics[width=\textwidth]{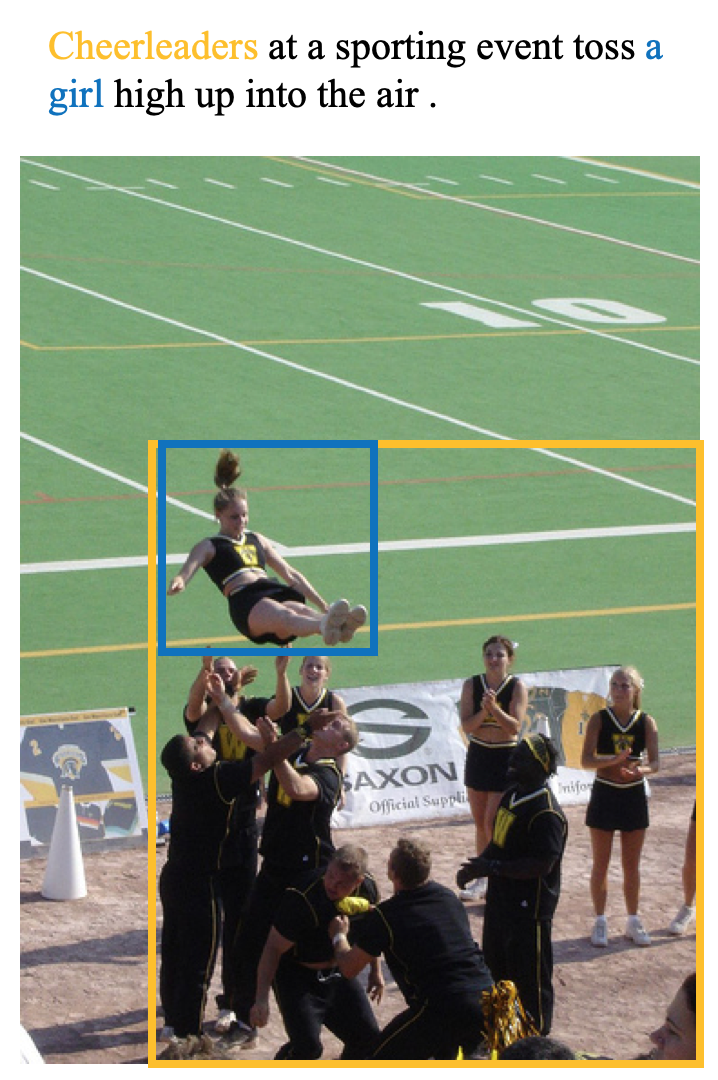}
        \caption{Dependency between entities. The visual relationship between grounding regions for \textcolor{orange}{``cheerleaders''} and \textcolor{blue}{``a girl''} should agree with context ``toss ... high up into the air''. }
        \label{fig:vis-rel}
    \end{subfigure}
    \begin{subfigure}[b]{0.48 \columnwidth}
        \includegraphics[width=\textwidth]{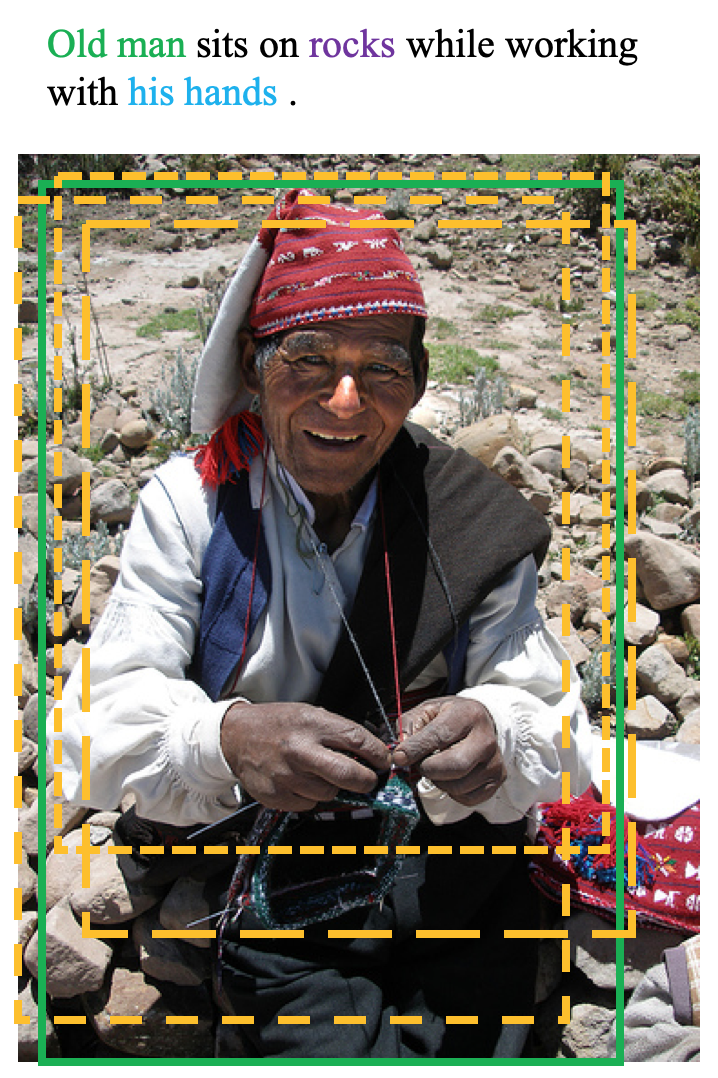}
        \caption{Gold label multiplicity. The \textcolor{green}{green box} is the annotated gold grounding region for entity phrase \textcolor{green}{``Old man''}, while the \textcolor{orange}{orange dash boxes} are region proposals with IoU $\ge 0.5$ with gold. }
        \label{fig:gold-mul}
    \end{subfigure}
    \caption{Example image-caption pairs from Flickr30k Entities, illustrating entity dependencies and gold label multiplicity. }
\end{figure}

\begin{figure}[t]
    \centering
    \begin{subfigure}[b]{0.48 \columnwidth}
        \includegraphics[width=\textwidth]{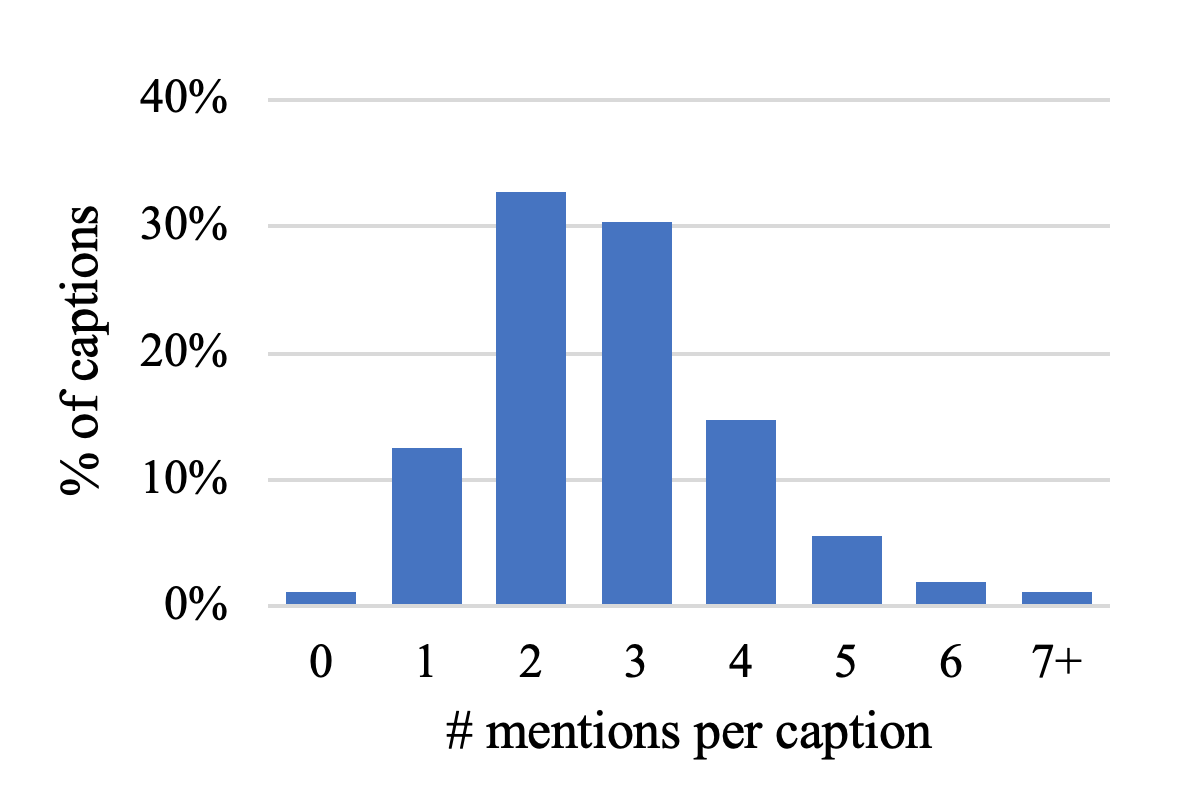}
        \caption{Distribution of number of entity phrases per caption. }
        \label{fig:men-cap}
    \end{subfigure}
    \begin{subfigure}[b]{0.48 \columnwidth}
        \includegraphics[width=\textwidth]{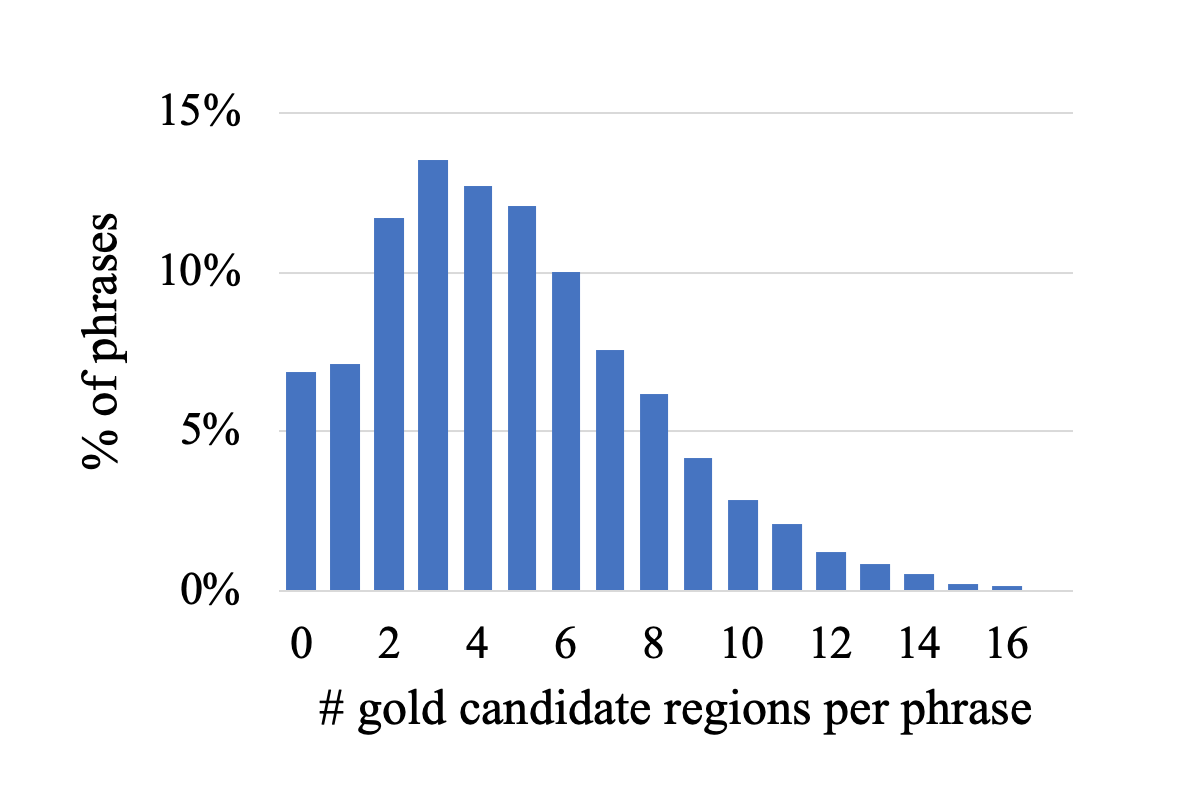}
        \caption{Distribution of number of gold labels per entity phrase. }
        \label{fig:reg-men}
    \end{subfigure}
    \caption{Validation set statistics for Flickr30k Entities. }
\end{figure}

Given an image and a corresponding caption, the phrase grounding task aims to ground each entity mentioned by a noun phrase in the caption to a region in the image. Phrase grounding has attracted much research interest due to its application in downstream tasks including image captioning \cite{DBLP:conf/nips/KarpathyJL14, DBLP:conf/cvpr/FangGISDDGHMPZZ15, DBLP:journals/pami/DonahueHRVGSD17, DBLP:conf/icml/XuBKCCSZB15}, image retrieval \cite{DBLP:conf/cvpr/ChenBFWN17, DBLP:conf/eccv/RadenovicTC16}, and visual question answering \cite{DBLP:journals/ijcv/AgrawalLAMZPB17, DBLP:conf/iccv/YuY0T17, DBLP:journals/tnn/YuYXFT18}. 

Phrase grounding systems typically work by ranking a set of candidate regions \cite{DBLP:conf/iccv/ChenKN17, DBLP:conf/ijcai/YuYXZ0T18}. Region proposals are generated from the image by a vision backbone model, without conditioning on the caption. Features of the phrase to be grounded are extracted, and subsequently interact with features of candidate regions, to determine phrase-region compatibility. Candidate regions are then ranked based on this compatibility metric, and the highest-scored candidate region is selected as the predicted grounding of the phrase. 

In Flickr30k Entities \cite{DBLP:journals/ijcv/PlummerWCCHL17}, each caption contains an average of 2.76 entity phrases to ground (Figure~\ref{fig:men-cap}; phrases with no corresponding gold regions are not counted). It therefore stands to reason that phrases in the same caption should not be grounded independently (to optimize each individual phrase-region assignment), but jointly (to optimize the global phrase-region assignment for the entire caption). Figure~\ref{fig:vis-rel} illustrates this phenomenon. The caption contains a sequence of two entity phrases, ``cheerleaders'' and ``a girl'', and the task is to label each phrase with a candidate region that best grounds it. Since there are several women present in the image, ``a girl'' has ambiguous grounding by itself, but it can be disambiguated by encouraging the visual relationship between ``a girl'' and ``cheerleaders'' to conform with context provided in the caption. 

Some works are aware that dependencies between entities in the same caption play an important role in building more accurate phrase grounding systems \cite{DBLP:conf/eccv/WangAKMD16, DBLP:conf/iccv/PlummerMCHL17, DBLP:conf/iccv/ChenKN17}. The success of these structured prediction methods shows the advantage of considering entity dependencies in learning and prediction. However, these approaches capture certain relations in an \textit{ad hoc} manner, and resort to approximate inference \cite{DBLP:conf/eccv/WangAKMD16, DBLP:conf/iccv/PlummerMCHL17} or non-differentiable losses \cite{DBLP:conf/iccv/ChenKN17}. 

To obtain models and inference algorithms that facilitate more globally consistent phrase grounding predictions, we propose to formulate phrase grounding as a sequence labeling task where we treat candidate regions as potential labels for the phrases in the input sequence. This allows us to build phrase grounding models based on Conditional Random Fields (CRFs) \cite{DBLP:conf/icml/LaffertyMP01} that capture entity dependencies in a universal and differentiable manner.  Our results indicate that systems that capture dependencies between phrases in the same caption in a principled manner outperform systems that ignore these dependencies. 

A second problem lies in the use of region proposals, which distinguishes phrase grounding from other sequence labeling tasks where CRFs are directly applicable. Following the metrics of object detection, in phrase grounding the correctness of a predicted region is judged by its overlap by Intersection-over-Union (IoU) with the gold region \cite{DBLP:journals/ijcv/PlummerWCCHL17}. To cover potential regions with high enough IoU, it is common to generate a myriad of region proposals and for these candidate regions to contain or substantially overlap with each other. As a result, there could be more than one candidate region with high IoU with the gold region, and they should all be considered as correct grounding for the phrase. This phenomenon of gold label multiplicity is illustrated in Figure~\ref{fig:gold-mul}. We hypothesize that it is important to consider gold label multiplicity and identify all correct region proposals during training, since the model would receive contradictory training signals if some correct proposals were marked as incorrect. With region proposals generated by a Bottom-Up Attention \cite{DBLP:conf/cvpr/00010BT0GZ18} visual backbone, in Flickr30k Entities each phrase has an average of 4.75 gold labels, and detailed statistics are presented in Figure~\ref{fig:reg-men}. To address this problem, we adopt the soft-label target distribution proposed by \citet{DBLP:conf/ijcai/YuYXZ0T18}, and our experiments show that models trained with this regime significantly outperform those trained with one-hot target regime. 

To combine the benefits brought by structured prediction from CRFs and by soft-label training regime, we define Soft-Label Chain CRFs, a variation of standard chain CRFs that allows us to work with gold label multiplicity. We adapt learning and inference algorithms from chain CRFs and develop an end-to-end training algorithm for our proposed model. 

We evaluate the effectiveness of Soft-Label Chain CRF on phrase grounding by conducting experiments on the Flickr30k Entities dataset \cite{DBLP:journals/ijcv/PlummerWCCHL17} and comparing grounding accuracy with strong baseline models, as well as with existing structured prediction methods and current state-of-the-art models. Experimental results show that our Soft-Label Chain CRF model outperforms its hard-label CRF counterpart by 2.43\%, a vanilla non-CRF soft-label model by 0.40\%, and the previous best results by about 1.4\%, demonstrating that both of our contributions, modeling phrase grounding as a sequence labeling task, and training with soft label targets, matter for this task. 

\section{Related Work}
\label{sec:related-work}

\smallsection{Phrase Grounding} The phrase grounding task was first postulated by \citet{DBLP:journals/pami/KarpathyF17} and \citet{DBLP:journals/ijcv/PlummerWCCHL17}, both of which moved from the holistic image captioning to the finer-grained task of matching regions with phrases in the caption. Datasets for this task include Flickr30k Entities \cite{DBLP:journals/ijcv/PlummerWCCHL17}, RefCOCO \cite{DBLP:conf/eccv/YuPYBB16}, and Visual Genome \cite{DBLP:journals/ijcv/KrishnaZGJHKCKL17}. The general framework of proposal-generation-ranking has become  adopted by most approaches to phrase grounding, and research in this area has focused on improving specific components of this framework. Our work can be viewed as an improvement to the training and prediction aspects. 

\smallsection{Structured Prediction in Phrase Grounding} We summarize some works that consider entity dependencies by structured prediction. Structured Matching \cite{DBLP:conf/eccv/WangAKMD16} formulates phrase grounding as a bipartite matching process between phrases and candidate regions, and encourages the spatial relationship between two grounding regions to conform to an extracted \emph{partial coreference} relation between their corresponding phrases. The resulting discrete optimization problem is then relaxed into a linear program to enable end-to-end training. Phrase-Region CCA \cite{DBLP:conf/iccv/PlummerMCHL17} mines \emph{frequent patterns} of semantically related paired phrases and trains a separate model for each pattern. The addition of this pairwise score makes the optimization a quadratic programming problem that requires approximate inference. QRC Net \cite{DBLP:conf/iccv/ChenKN17} assumes that phrases in a caption refer to distinct entities, and thus predicted grounding regions are penalized for \emph{spatial overlapping}. However, overlapping regions can be penalized only after prediction, so this loss is not differentiable, and one has to resort to reinforcement learning. In these works, partial coreference extraction, frequent patterns mining and spatial overlap penalties are \textit{ad hoc} entity dependency capturing, while we aim to universally encompass the spectrum of such dependencies. 

\smallsection{Soft-Label Training Regime} Conventionally, region proposal ranking is done by predicting a probability distribution over all candidate regions for grounding a given entity phrase, which is learned to match a target distribution. \citet{DBLP:conf/iccv/ChenKN17} and \citet{DBLP:conf/eccv/RohrbachRHDS16} define the target distribution as a one-hot vector which only gives credit to the candidate region with highest IoU with the gold region, and cross-entropy loss is used as training objective. Under this hard-label training regime, the model is trained to pick only the best candidate region while rejecting all the inferior-than-best candidate regions, which is intuitively not a good behavior. \citet{DBLP:conf/ijcai/YuYXZ0T18} proposes a soft-label target distribution which gives weighted credit to all good candidate regions (i.e. those with above-threshold IoU with the gold region), and uses Kullback-Leibler (KL) divergence loss as training objective. 

\smallsection{Conditional Random Fields} CRFs \cite{DBLP:conf/icml/LaffertyMP01} are discriminative probabilistic models that have been found useful in sequence labeling tasks by capturing label dependencies \cite{DBLP:conf/acl/MaH16, DBLP:conf/naacl/LampleBSKD16}. We summarize some works relevant to CRFs learned in soft-label or multi-label settings. Multi-CRFs \cite{dredze2009sequence} learn CRFs with noisy annotated data, where annotators may disagree on the label for input tokens. The assumption is that there is always only one gold label for each token, so the model favors single label while conforming to the prior distribution of labels set by annotators. To work with soft-label targets, it employs a mode-seeking, exclusive KL divergence definition, which does not imply moment-matching, a desired property of CRFs (and in general, exponential family models) that we show in Section~\ref{ssec:sl-crf} and \ref{ssec:sl-ccrf} for the mean-seeking, inclusive KL divergence definition in our model. \citet{DBLP:journals/ml/RodriguesPR14} models the latent reliability of individual annotators, and use this information to guide the selection of trustworthy annotation sources and estimation of real gold labels. Note that both works always assume one gold label per input token, where the ambiguity comes from unreliability of annotations, while our work focuses on cases where there may be multiple gold labels per input token by the nature of the task. 

\section{Soft-Label Chain CRF}
\label{sec:sl-ccrf}

CRFs model the probability of a label sequence $\bm{y}\!=\!y^{1:T}$ conditioned on an input sequence $\bm{x}\!=\!x^{1:T}$ in terms of a score function $s(\bm{x},\bm{y})$: 
\begin{align*}
	p(\bm{y} | \bm{x})
		= \frac{\exp{s(\bm{y}, \bm{x})}}{\sum_{\bm{y}'}{\exp{s(\bm{y}', \bm{x})}}}
\end{align*}

For a given training example $\{(\bm{x}, \bm{y})\}$, the negative log-likelihood loss (i.e. cross-entropy loss w.r.t. a one-hot target distribution that gives credit to the gold label only) is
\begin{align*}
	L
		= -\log{p(\bm{y} | \bm{x})}
		= -s(\bm{y}, \bm{x}) + \log{Z(\bm{x})}
\end{align*}
where $Z(\bm{x}) = \sum_{\bm{y}'}{\exp{s(\bm{y}', \bm{x})}}$. The gradient of this loss w.r.t. score function is
\begin{align*}
	\frac{\partial{L}}{\partial{s(\bm{y}', \bm{x})}}
		= -\mathbb{I}(\bm{y}' = \bm{y}) + p(\bm{y}' | \bm{x})
\end{align*}
which is known as moment-matching. This allows us to train CRFs with gradient methods and conveniently connect to backpropagation when the score function is modeled by a neural architecture. 

\begin{algorithm*}
\caption{Modified forward algorithm to compute the KL divergence loss for Soft-Label Chain CRFs}
\begin{algorithmic}
\Procedure{SoftLabelChainCrfLoss}{$\bm{q}, \eps(y^t, \bm{x}), \tau(y^t, y^{t-1}, \bm{x})$}
	\ForAll {label $y^0$}
		\State $\alpha^0_{y^0} \gets 0$
		\State $g^0_{y^0} \gets 0$
	\EndFor
	\For {$t = 1 \hdots T$}
		\ForAll {label $y^t$}
			\State $\alpha^t_{y^t} \gets \sum_{y^{t-1}}{\Big\{\alpha^{t-1}_{y^{t-1}} \exp{\big[\tau(y^t, y^{t-1}, \bm{x}) + \eps(y^t, \bm{x})\big]}\Big\}}$
			\State $g^t_{y^t} \gets \sum_{y^{t-1}}{\Big\{\big[g^{t-1}_{y^{t-1}} + \tau(y^t, y^{t-1}, \bm{x})\big] q^{t-1}_{y^{t-1}} + \big[\eps(y^t, \bm{x}) - \log{q^t_{y^t}}\big]\Big\}}$
		\EndFor
	\EndFor
	\State $Z \gets \sum_{y^T}{\alpha^T_{y^T}}$
	\State $G \gets \sum_{y^T}{g^T_{y^T} q^T_{y^T}}$
	\State $L \gets -G + \log{Z}$
	\State \Return $L$
\EndProcedure
\end{algorithmic}
\label{alg:fwd}
\end{algorithm*}

\subsection{Soft-Label CRF}
\label{ssec:sl-crf}

In the standard CRF above, each input $x^t$ corresponds to a single gold label $y^t$. To account for gold label multiplicity in training stage, we replace the sequence of gold labels $\bm{y}$ with a sequence of distributions $\bm{q} = q^{1:T}$ where $q^t \in \mathbb{R}^K$ is the gold label distribution over all $K$ possible labels for input $x^t$. Note that this distribution should not be interpreted as the confidence of each label being correct; rather, it should be understood as a probabilistic gold label model: if we randomly choose a gold label, how likely is each label to be selected. With independence assumption, the gold probability of an arbitrary label sequence $\bm{y}$ is
\begin{align*}
	q(\bm{y} | \bm{x})
		= \prod_{t}{q(y^t | \bm{x})}
		= \prod_{t}{q(y^t | x^t)}
		= \prod_{t}{q^t_{y^t}}
\end{align*}
It is easy to see that $q(\bm{y} | \bm{x})$ is a distribution: 
\begin{align*}
    \sum_{\bm{y}}{q(\bm{y} | \bm{x})}
        = \sum_{\bm{y}}{\prod_t{q^t_{y^t}}}
        = \prod_t{\sum_{y^t}{q^t_{y^t}}}
        = 1
\end{align*}
And our goal is to learn this target distribution. 

Since this target distribution is no longer degenerate, we use Kullback-Leibler (KL) divergence to measure the discrepancy between the model and the target distribution. Our training objective is the KL divergence loss (in mean-seeking, inclusive form): 
\begin{align*}
	L
		= \sum_{\bm{y}}{\Big\{q(\bm{y} | \bm{x}) \log{\frac{q(\bm{y} | \bm{x})}{p(\bm{y} | \bm{x})}}\Big\}}
\end{align*}
which also gives gradients that demonstrate moment-matching: 
\begin{align*}
	\frac{\partial{L}}{\partial{s(\bm{y}', \bm{x})}}
		= -q(\bm{y}' | \bm{x}) + p(\bm{y}' | \bm{x})
\end{align*}
Note that if we had defined the KL divergence loss in its mode-seeking, exclusive form $\sum_{\bm{y}}{p(\bm{y} | \bm{x}) \log{\frac{p(\bm{y} | \bm{x})}{q(\bm{y} | \bm{x})}}}$, we would have lost this desired moment-matching property. 

\subsection{Factorization of Soft-Label Chain CRF}
\label{ssec:sl-ccrf}

Learning CRFs of general graphs requires inference in unit of cliques, which is usually computationally intractable. By restricting to local, pairwise potentials, we reduce the model to a first-order linear chain CRF, whose scoring function factorizes as
\begin{align*}
	s(\bm{y}, \bm{x})
		&= \sum_t{s(y^t, y^{t-1}, \bm{x})} \\
		&= \sum_t{\Big\{\tau(y^t, y^{t-1}, \bm{x}) + \eps(y^t, \bm{x})\Big\}}
\end{align*}
where $\tau(\cdot, \cdot, \cdot)$ is the transition score between labels at $t-1$ and $t$ that captures the dependency between labels for adjacent input tokens, and $\eps(\cdot, \cdot)$ is the emission score between label and input at $t$. 

Combining this factorization with soft-label targets gives the formal definition of Soft-Label Chain CRF. The loss can be written as
\allowdisplaybreaks
\begin{align*}
L\!
	&= \!\sum_{\bm{y}}\!{\Big\{q(\bm{y} | \bm{x}) \log{\frac{q(\bm{y} | \bm{x})}{p(\bm{y} | \bm{x})}}\!\Big\}} \\
	&= \!\sum_{\bm{y}}\!{\Big\{\!q(\bm{y} | \bm{x}) \big[\!\log{q(\bm{y} | \bm{x})} \!-\! s(\bm{y},\!\bm{x})\!+\!\log{Z(\bm{x})}\!\big]\!\Big\}} \\
	&\qquad (\text{Expand } p(\bm{y} | \bm{x}) \text{ by CRF modeling}) \\
	&= \!\sum_{\bm{y}}\!{\Big\{\!q(\bm{y} | \bm{x}) \big[\!\log{q(\bm{y} | \bm{x})} \!-\!s(\bm{y},\!\bm{x})\!\big]\!\Big\}}\!+ \!\log{Z(\bm{x})} \\
	&\qquad (\text{Marginalize } q(\bm{y} | \bm{x})) \\
	&= \!\sum_t\!{\sum_{y^t}\!{\Big\{\!q(y^t | \bm{x}) \log{q(y^t | \bm{x})}\!\Big\}}} \\
	&\qquad \!-\!\sum_{\bm{y}}\!{\Big\{\!q(\bm{y} | \bm{x}) s(\bm{y},\!\bm{x})\!\Big\}}\!+\!\log{Z(\bm{x})} \\
	&\qquad (\text{Independence of } q(y^t | \bm{x}) \text{ across } t) \\
	&= \!\sum_t\!{\sum_{y^t}\!{\Big\{\!q(y^t | \bm{x}) \log{q(y^t | \bm{x})}\!\Big\}}} \\
	&\qquad \!-\!\sum_{\bm{y}}\!{\Big\{\!q(\bm{y} | \bm{x}) \sum_t\!{s(y^t,\!y^{t-1},\!\bm{x})}\!\Big\}}\!+\!\log{Z(\bm{x})} \\
	&\qquad (\text{Factorization of } s(\bm{y},\!\bm{x})) \\
	&= \!\sum_t\!{\sum_{y^t}\!{\Big\{\!q(y^t | \bm{x}) \log{q(y^t | \bm{x})}\!\Big\}}} \\
	&\qquad \!-\!\sum_t\!{\sum_{y^t,\!y^{t-1}}\!{\Big\{\!s(y^t,\!y^{t-1},\!\bm{x}) \sum_{\bm{y} | y^t,\!y^{t-1}}\!{q(\bm{y} | \bm{x})}\!\Big\}}} \\
	&\qquad \!+\!\log{Z(\bm{x})} \\
	&\qquad (\text{Reorganize sums by } s(y^t,\!y^{t-1},\!\bm{x})) \\
	&= \!\sum_t\!{\sum_{y^t}\!{\Big\{\!q(y^t | \bm{x}) \log{q(y^t | \bm{x})}\!\Big\}}} \\
	&\qquad \!-\!\sum_t\!{\sum_{y^t,\!y^{t-1}}\!{\Big\{\!q(y^t,\!y^{t-1} | \bm{x}) s(y^t,\!y^{t-1},\!\bm{x})\!\Big\}}} \\
	&\qquad \!+\!\log{Z(\bm{x})}
\end{align*}
which gives moment-matching gradients
\begin{align*}
	\frac{\partial{L}}{\partial{s(y^t, y^{t-1}, \bm{x})}}
		&= -q(y^t, y^{t-1} | \bm{x}) + p(y^t, y^{t-1} | \bm{x}) \\
	\frac{\partial{L}}{\partial{\tau(y^t, y^{t-1}, \bm{x})}}
		&= -q(y^t, y^{t-1} | \bm{x}) + p(y^t, y^{t-1} | \bm{x}) \\
	\frac{\partial{L}}{\partial{\eps(y^t, \bm{x})}}
		&= -q(y^t | \bm{x}) + p(y^t | \bm{x})
\end{align*}
where
\begin{align*}
	q(y^t | \bm{x})
	    &= q^t_{y_t} \\
	q(y^t, y^{t-1} | \bm{x})
	    &= q^t_{y^t} q^{t-1}_{y^{t-1}}
\end{align*}
are the probability of local label(s) marginalized over all possible non-local labels. Smoothing inference $p(y^t | \bm{x})$ and $p(y^t, y^{t-1} | \bm{x})$ can be computed with forward-backward algorithm. 

\subsection{As an Extension of Soft-Label Model}
\label{ssec:sl}

Note that if we omit all transition terms in Soft-Label Chain CRF, the loss reduces to
\begin{align*}
	L'
		&= \sum_t{\sum_{y^t}{\Big\{q(y^t | \bm{x}) \big[{-\eps(y^t, \bm{x})} + \log{q(y^t | \bm{x})}\big]\Big\}}} \\
		&\qquad + \log{Z(\bm{x})} \\
		&= \sum_t{\sum_{y^t}{\Big\{q(y^t | \bm{x}) \log{\frac{q(y^t | \bm{x})}{p(y^t | \bm{x})}}\Big\}}}
\end{align*}
which is a total factorization over time. This is as if each label is predicted independently using a soft-label training regime, which is exactly the KL divergence loss proposed by \citet{DBLP:conf/ijcai/YuYXZ0T18}. Therefore, our Soft-Label Chain CRF can be viewed as an extension of this soft-label discriminative model. 

\subsection{Modified Forward Algorithm}

For chain CRFs, computing the loss only requires forward algorithm, while computing the gradients requires a full forward-backward algorithm. It can be proved that backpropagation on the loss gives the same result as running forward-backward. This is a commonly used trick in modern deep learning frameworks to eliminate the need of implementing the backward pass. Algorithm~\ref{alg:fwd} presents a modified forward algorithm that computes the loss for Soft-Label Chain CRF. In Section~1 and 2 of the Supplementary Materials, we prove the correctness of this algorithm, and that its backpropagation is also equivalent to forward-backward. 

\section{Phrase Grounding as Sequence Labeling}
\label{sec:pg-sl}

\subsection{Task Formulation}

We formulate phrase grounding as a sequence labeling task. Given an image $I$, a caption sentence $[c^1 \hdots c^L]$ where $c^l$ is a word token, and a set of non-overlapping noun phrase spans $[p^1 \hdots p^T]$ where $p^t = (s^t, e^t)$ denotes that the $t$'th phrase covers tokens $c^{s^t}$ to $c^{e^t}$ (inclusive), we generate a set of region proposals $\{r_1 \hdots r_K\}$, label each phrase with a candidate region, and refine the region by performing a bounding box regression. 

\subsection{Model Specification}
\label{ssec:model}

\begin{figure}[t]
    \centering
    \includegraphics[width=\columnwidth]{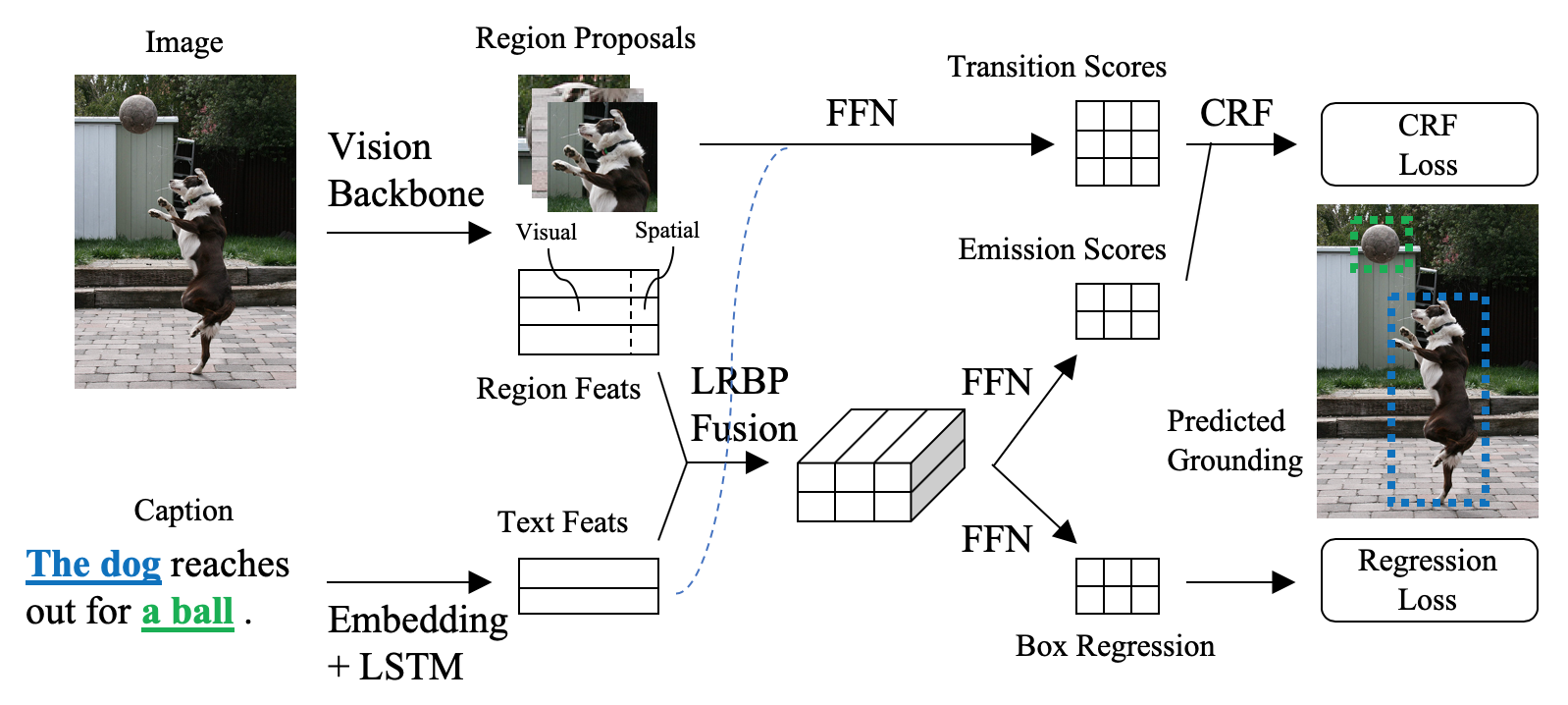}
    \caption{Our model for phrase grounding as a sequence labeling task. The $K\!\times\!K$ transition score matrix is derived from the features of $K$ region proposals. The $T\!\times\!K$ emission score matrix is derived from a joint representation of phrase-region pairs, which is fused from features of region proposals and $T$ entity phrases. Bounding box regression is applied to the sequence of regions predicted by the CRF. \textcolor{cyan}{Cyan dashed line}: contextualized transition score prediction (Section~\ref{ssec:model}). }
    \label{fig:model}
\end{figure}

\begin{figure}[t]
    \centering
    \includegraphics[width=\columnwidth]{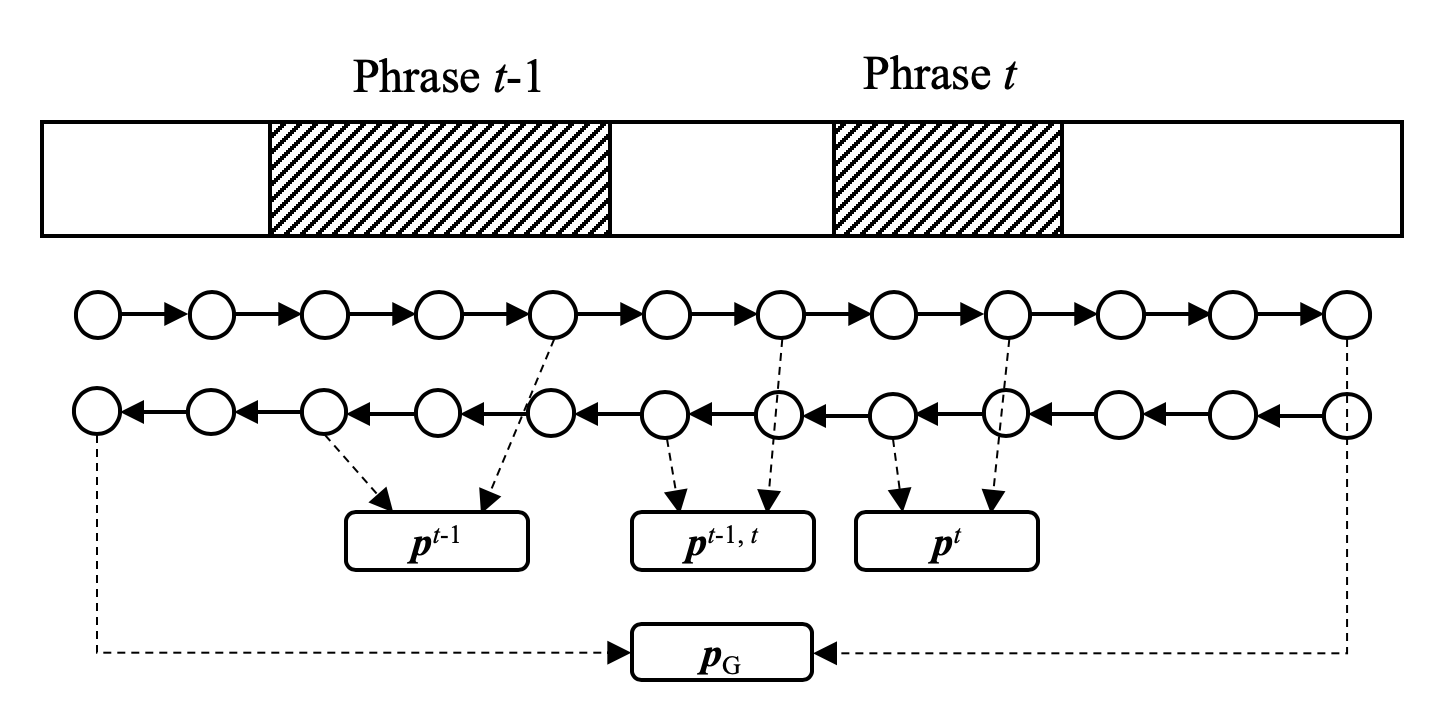}
    \caption{Text feature extraction for phrases in a caption. Shaded regions are entity phrase spans; circles represent LSTM cells. For phrase $t$ hidden states at its span boundaries are concatenated to form its text features $\bm{p}^t$, which is used in fusion with region features. For the contextualized transition score between phrases $t-1$ and $t$, hidden states at the boundaries of the context between them are concatenated into a context feature vector $\bm{p}^{t-1,t}$, which can be further extended by phrase features $\bm{p}^{t-1}$ and $\bm{p}^{t}$ as well as global text features $\bm{p}_G$. }
    \label{fig:textfeat}
\end{figure}

Figure~\ref{fig:model} outlines our phrase grounding model. $K$ region proposals and their visual and spatial features are extracted from an object detection vision backbone. We feed the token embeddings of the caption into a bi-directional LSTM \cite{DBLP:journals/neco/HochreiterS97}, and then concatenate the forward hidden state at the ending boundary of the phrase with the backward hidden state at the starting boundary of the phrase (see Figure~\ref{fig:textfeat}). This phrase representation captures context both preceding and following the phrase in the caption. 
\begin{align*}
	(\overrightarrow{\bm{h}^{1:L}}, \overleftarrow{\bm{h}^{1:L}})
		&= \text{BiLSTM}(\text{Embed}([c^1 \hdots c^L])) \\
	\bm{p}^t
		&= [\overrightarrow{\bm{h}^{e^t}} || \overleftarrow{\bm{h}^{s^t}}]
\end{align*}
We use low-rank bilinear pooling (LRBP) \cite{DBLP:conf/iclr/KimOLKHZ17} to fuse text and region features. Compared to simple concatenation, LRBP supports pairwise interaction between bimodal feature channels while keeping a reasonable computation overhead. Given a text feature vector $\bm{p}^t \in \mathbb{R}^{d_{\text{text}}}$ and a region feature vector $\bm{r}_k \in \mathbb{R}^{d_{\text{vis}}}$, LRBP fuses them into a joint representation $\bm{f}^t_k \in \mathbb{R}^{d_{\text{joint}}}$: 
\begin{align*}
	\bm{f}^t_k
		&= \bm{P}^\top (\bm{U}^\top \bm{p}^t \circ \bm{V}^\top \bm{r}_k) + \bm{b}
\end{align*}
where $\bm{U} \in \mathbb{R}^{d_{\text{text}} \times r}$, $\bm{V} \in \mathbb{R}^{d_{\text{vis}} \times r}$, pooling matrix $\bm{P} \in \mathbb{R}^{r \times d_{\text{joint}}}$, bias $\bm{b} \in \mathbb{R}^{d_{\text{joint}}}$, and $\circ$ is the Hadamard (i.e. element-wise) product. 

As discussed in Section~\ref{ssec:sl-ccrf}, the CRF score function consists of emission score and transition score. The emission score $\eps(r_k, p^t)$ models the compatibility between each phrase and each candidate region. We feed the joint representation to a single-layer feed-forward neural network: 
\begin{align*}
    \eps(r_k, p^t)
        &= \text{FFN}(\bm{f}^t_k)
\end{align*}
The transition score $\tau(r_k, r_{k'}, p^{1:T})$ is modeled by a two-layer feed-forward neural network with ReLU activation for the hidden layer: 
\begin{align*}
    \tau(r_k, r_{k'}, p^{1:T})
        &= \text{FFN}(\sigma(\text{FFN}([\bm{r}_k || \bm{r}_{k'}])))
\end{align*}
To condition the transition scores on local and global context from the caption, we can extend the input $[\bm{r}_k || \bm{r}_{k'}]$ with the following text features: context in between the two phrases (feature vector $\bm{p}^{t-1,t}$), context from phrase features $\bm{p}^{t-1}$ and $\bm{p}^t$, and global context $\bm{p}_G$. 

One important difference between the standard use of CRFs for sequence labeling and our task is that our "labels" do not correspond to a fixed set of classes that can be predicted for any input, but are as specific to the particular input example as the sequences to be labeled themselves. Hence, our transition and emission scores do not depend on the (arbitrary) indices of regions to be ground, but on their visual and spatial features (as well as on their corresponding linguistic contexts). Finally, although our approach could in principle be extended to higher-order CRFs, we restrict our attention here to first-order CRFs for computational efficiency. As a consequence, our models can only capture dependencies between string-adjacent phrases. 

\subsection{Training Objectives}

For each image-caption instance, the loss is a linear combination of the labeling and bounding box regression loss: 
\begin{align*}
	L &= L_{\text{label}} + \gamma L_{\text{reg}}
\end{align*}
$L_{\text{label}}$ is the CRF loss defined in Section~\ref{ssec:sl-ccrf}. $L_{\text{reg}}$  \cite{DBLP:journals/pami/RenHG017} is defined as
\begin{align*}
	L_{\text{reg}} &= (\beta, \hat{\beta}) &= \sum_{i \in \{x,y,w,h\}}{\text{SmoothL1}(\hat{\beta_i} - \beta_i)}
\end{align*}
with the ground truth regression parameterization
\begin{align*}
    \beta
        &= [\frac{x-x_a}{w_a}, \frac{y-y_a}{h_a}, \log{\frac{w}{w_a}}, \log{\frac{h}{h_a}}]
\end{align*}
and 
\begin{align*}
    \text{SmoothL1}(x) &= \begin{cases} 0.5x^2 & \text{if } |x| < 1 \\ |x|-0.5 & \text{otherwise} \end{cases}
\end{align*}

\section{Experiments}
\label{sec:experiments}

\begin{table*}[t]
    \centering
    \resizebox{\textwidth}{!}{
    \begin{tabular}{l l c}
        \hline
        \textbf{Method} & \textbf{Vision Backbone} & \textbf{Grounding Accuracy (\%)} \\
        \hline
        \emph{Compared methods} \\
        Structured Matching \cite{DBLP:conf/eccv/WangAKMD16} & Fast R-CNN \cite{DBLP:conf/iccv/Girshick15} & 42.08 \\
        Phrase-Region CCA \cite{DBLP:conf/iccv/PlummerMCHL17} & Fast R-CNN \cite{DBLP:conf/iccv/Girshick15} & 55.85 \\
        QRC Net \cite{DBLP:conf/iccv/ChenKN17} & Fast R-CNN \cite{DBLP:conf/iccv/Girshick15} & 65.14 \\
        BAN \cite{DBLP:conf/nips/KimJZ18} & Bottom-Up Attention \cite{DBLP:conf/cvpr/00010BT0GZ18} & 69.69 \\
        DDPN \cite{DBLP:conf/ijcai/YuYXZ0T18} & Bottom-Up Attention \cite{DBLP:conf/cvpr/00010BT0GZ18} & 73.3 \\
        \hline
        \emph{Our methods} \\
        Hard-Label (GloVe \cite{DBLP:conf/emnlp/PenningtonSM14}) & Bottom-Up Attention \cite{DBLP:conf/cvpr/00010BT0GZ18} & 71.88 \\
        Hard-Label (HL) & Bottom-Up Attention \cite{DBLP:conf/cvpr/00010BT0GZ18} & 72.21 \\
        Soft-Label (SL) & Bottom-Up Attention \cite{DBLP:conf/cvpr/00010BT0GZ18} & 74.29 \\
        Hard-Label Chain CRF (HL-CCRF) & Bottom-Up Attention \cite{DBLP:conf/cvpr/00010BT0GZ18} & 72.26 \\
        Soft-Label Chain CRF (SL-CCRF) & Bottom-Up Attention \cite{DBLP:conf/cvpr/00010BT0GZ18} & \textbf{74.69} \\
        \hline
    \end{tabular}
    }
    \caption{Performance of different phrase grounding methods on Flickr30k Entities (test set). Our CRF models has transition scores conditioned on features of context in between the two phrases (``M'' in Table~\ref{tab:acc-tran}). Our methods, unless explicitly specified, uses ELMo \cite{DBLP:conf/naacl/PetersNIGCLZ18} as word embeddings. }
    \label{tab:acc}
\end{table*}

\subsection{Experiment Setup}

\smallsection{Dataset} We train and evaluate our models on the Flickr30k Entities dataset \cite{DBLP:journals/ijcv/PlummerWCCHL17}, which contains $31,783$ images, each accompanied by $5$ captions. In keeping with previous work on this dataset, we assume that entity phrase boundaries are given, so inferring which phrases to ground is not part of our task. Following \citet{DBLP:journals/ijcv/PlummerWCCHL17}, we merge all regions that are ground to the same phrase into one larger bounding box, and split the dataset into $29,783$ training images, $1k$ validation images and $1k$ test images. 

We do not apply our method to RefCOCO \cite{DBLP:conf/eccv/YuPYBB16} or Visual Genome \cite{DBLP:journals/ijcv/KrishnaZGJHKCKL17} because they consist of independently grounded entity phrases without any entity dependencies that CRFs could leverage. 

\smallsection{Implementation details} For text feature extraction, we use the 1024-d contextualized word embeddings from the last layer of ELMo \cite{DBLP:conf/naacl/PetersNIGCLZ18}, followed by a bi-directional LSTM \cite{DBLP:journals/neco/HochreiterS97} encoder with hidden dimension $d_{\text{hidden}}\!=\!512$ for each direction, so that the text feature vector has dimension $d_{\text{text}}\!=\!1024$. We use the Bottom-Up Attention model \cite{DBLP:conf/cvpr/00010BT0GZ18} to generate region proposals and extract visual features, as in the state-of-the-art BAN \cite{DBLP:conf/nips/KimJZ18} and DDPN \cite{DBLP:conf/ijcai/YuYXZ0T18} models. $K\!=\!100$ region proposals are generated for each image. Each candidate region with coordinates $(x_{\text{min}}, y_{\text{min}})$,  $(x_{\text{max}}, y_{\text{max}})$ is represented by a $d_{\text{vis}} = 2053$ feature vector that consists of 2048-d visual features concatenated with 5-d spatial features $[x_{\text{min}}/W, y_{\text{min}}/H, x_{\text{max}}/W, y_{\text{max}}/H, wh/WH]$. The low-rank bilinear pooling (LRBP) layer used for text-region bimodal feature fusion  has rank $r\!=\!1024$ and output dimension $d_{\text{joint}}\!=\!1024$. We train with a mini-batch size of $16$ image-caption instances. Each instance contains all entity phrases to be grounded in the caption. Weights  are initialized with  Xavier  \cite{DBLP:journals/jmlr/GlorotB10}. We apply a dropout rate of $p\!=\!0.2$ after the word embedding layer, LSTM layer, and LRBP fusion layer. The loss weighting parameter $\gamma$ is $10.0$. All gradients are clipped by $\infty$-norm of $10.0$ to prevent gradient explosion. We do not fine-tune ELMo or the Bottom-Up Attention model. All models are trained for $50k$ iterations using  Adam  \cite{DBLP:journals/corr/KingmaB14} with learning rate $5e\!-\!5$ and $\beta_1\!=\!0.9, \beta_2\!=\!0.98$. Model snapshots are taken every $5k$ iterations and the model with the highest validation set accuracy is selected. 

\smallsection{Metrics} We predict one grounded region for each entity phrase. Following \citet{DBLP:journals/ijcv/PlummerWCCHL17}, a prediction is deemed accurate if it has at least 0.5 IoU overlap with the gold region. We report the percentage of accurately grounded phrases. 

\subsection{Quantitative Results}

\begin{table}
    \centering
    \resizebox{\columnwidth}{!}{
    \begin{tabular}{c c c}
        \hline
        \textbf{Model} & \textbf{Transition Context} & \textbf{Accuracy (\%)} \\
        \hline
        SL-CCRF & -- & 74.28 \\
        SL-CCRF & M & \textbf{74.69} \\
        SL-CCRF & M+LR & 74.45 \\
        SL-CCRF & M+LR+G & 74.48 \\
        \hline
    \end{tabular}
    }
    \caption{Performance of Soft-Label Chain CRF models by conditioning transition scores on different sets of context features. --: input to transition score prediction is $[\bm{r}_k || \bm{r}_{k'}]$. M: input extended by features of context $\bm{p}^{t-1,t}$ in between the two phrases. M+LR: input further extended by features of LHS phrase $\bm{p}^{t-1}$ and RHS phrase $\bm{p}^t$. M+LR+G: input further extended by features of global context $\bm{p}_G$. }
    \label{tab:acc-tran}
\end{table}

\begin{table}
    \centering
    \resizebox{\columnwidth}{!}{
    \begin{tabular}{c c c}
        \hline
        \textbf{Decoding Algorithm} & HL-CCRF & SL-CCRF \\
        \hline
        Viterbi (MAP) & 72.26 & 74.69 \\
        Smoothing & \textbf{72.30} & \textbf{74.73} \\
        \hline
    \end{tabular}
    }
    \caption{Decoding algorithms' impact on performance. }
    \label{tab:acc-decode}
\end{table}

We compare our Soft-Label Chain CRF model against three baselines: a Hard-Label non-CRF model, a Hard-Label CRF, and a Soft-Label non-CRF model. The non-CRF models ground each phrase independently with a loglinear model. 
The Hard-Label models are trained with a standard one-hot training regime. The Soft-Label models use the soft-label training regime described above. The Soft-Label non-CRF model corresponds to the reduced form of the Soft-Label Chain CRF in Section~\ref{ssec:sl}. 

Table~\ref{tab:acc} shows the performance of previous structured prediction models, current state-of-the-art models, our baseline models and the Soft-Label Chain CRF model. For a fair comparison with BAN \cite{DBLP:conf/nips/KimJZ18}, we also report result of the hard-label baseline with GloVe \cite{DBLP:conf/emnlp/PenningtonSM14} embeddings, while we obtain 0.33\% higher result with ELMo. 
Training a non-CRF model on soft-label target distributions improves  accuracy by a further 2.08\%. On top of that, Soft-Label Chain CRF improves accuracy by another 0.40\%, which shows the effectiveness of treating phrase grounding as a sequence labeling task and using CRFs to capture entity dependencies. We also observe that the Hard-Label Chain CRF outperforms the hard-label baseline by a mere margin of 0.05\%, so our conjecture is that using chain CRFs works well only with a suitable choice of training regime. Soft-Label Chain CRF gives an overall improvement of 2.48\% over the hard-label baseline; it significantly outperforms previous structured prediction models including Structured Matching \cite{DBLP:conf/eccv/WangAKMD16}, Phrase-Region CCA \cite{DBLP:conf/iccv/PlummerMCHL17} and QRC Net \cite{DBLP:conf/iccv/ChenKN17}, and surpasses the state-of-the-art  BAN \cite{DBLP:conf/nips/KimJZ18} and DDPN \cite{DBLP:conf/ijcai/YuYXZ0T18} models by a margin of 5.00\% and about 1.4\%, respectively. 

We conduct an ablation study to find the most appropriate combination of context features for the transition scores in the SL-CCRF model. Table~\ref{tab:acc-tran} shows that we obtain the best results by including the context in between the two phrases, which gives an improvement of 0.41\%. We did not see any benefit from adding further text features from the left and right side of the phrases, or from the entire caption. 

Besides the Viterbi decoding algorithm used in prediction in CRFs, we also experiment with a smoothing decoding algorithm. While Viterbi finds the MAP label sequence conditioned on the input sequence $\argmax_{\bm{y}}{p(\bm{y} | \bm{x})}$, smoothing decoding finds the best label for each input $x^t$: $\argmax_{y^t}{p(y^t | \bm{x})}$. This makes sense in some scenarios where we want to refine the predicted grounding of one entity by referring to the context instead of attempting to ground all entities mentioned in the description. Table~\ref{tab:acc-decode} shows that in both Hard-Label Chain CRF and Soft-Label Chain CRF, smoothing decoding gives a prediction accuracy 0.04\% higher than Viterbi decoding. 

\begin{figure*}[t]
    \centering
    \includegraphics[width=1.0 \textwidth]{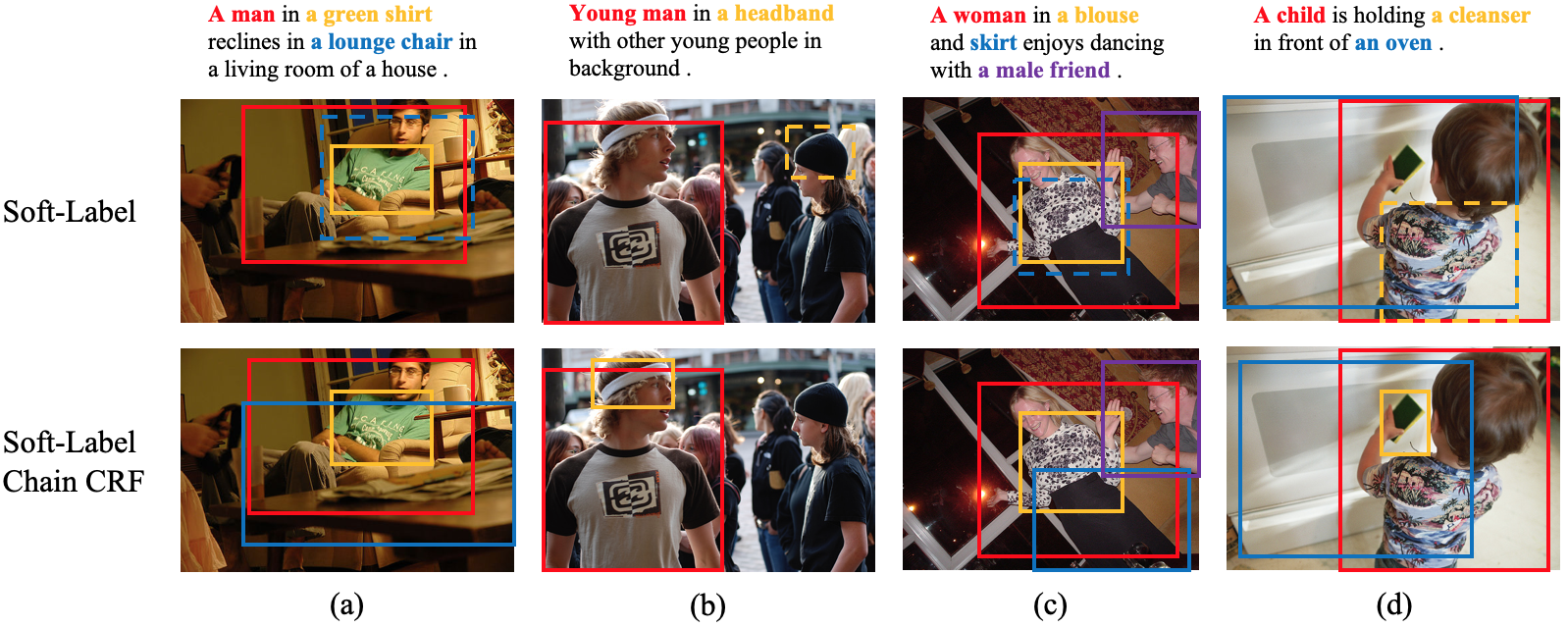}
    \caption{Selected visualization of phrase grounding results in the validation set of Flickr30k Entities. Solid boxes are correct predicted groundings, while dashed boxes are incorrect predicted groundings. Gold regions are not shown. Each entity phrase and its predicted grounding are marked with same color. Best viewed in color. }
    \label{fig:qual}
\end{figure*}

Without bounding box regression, the Soft-Label Chain CRF model has an accuracy of 69.85\%, a 4.84\% reduction compared to the setting with bounding box regression. 

\subsection{Qualitative Results}

We visualize some phrase grounding results in the validation set of Flickr30k Entities in Figure~\ref{fig:qual}. In (a), our CRF model avoids the error in grounding ``a lounge chair'' by constraining its relative position to ``a man''. In (b), although it may not have learned to distinguish ``headband'' and ``hat'', the CRF constrains the spatial position of ``headband'' to agree with the ownership dependency provided in context. In (c), it avoids the error in grounding ``skirt'' by spatially discriminating it from ``a blouse''. In (d), it avoids the error in grounding ``a cleanser'' by constraining its relative size w.r.t. ``a child''. These examples indicate that the CRF model may avoid grounding errors made by non-CRF models by leveraging entity dependencies, including relative position, spatial overlapping, and relative size. 

\section{Conclusion}
\label{sec:conclusion}

In this paper, we formulate phrase grounding as a sequence labeling task and propose the Soft-Label Chain CRF model that successfully combines the benefits brought by global structured prediction and soft-label training regime that addresses the gold label multiplicity problem. Experimental results show that we achieve an overall improvement of 2.48\% on grounding accuracy compared to a strong baseline, and that our model outperforms previous methods on phrase grounding. 

\section{Acknowledgements}
\label{sec:ack}
This material is based upon work supported by the National Science Foundation under Grants No. 1405883 and 1563727. Any opinions, findings, and conclusions or recommendations expressed in this material are those of the author(s) and do not necessarily reflect the views of the National Science Foundation. 

\bibliographystyle{acl_natbib}
\bibliography{emnlp-ijcnlp-2019}

\end{document}


\maketitle

\section{Correctness of the Modified Forward Algorithm for Soft-Label Chain CRFs}
\label{sec:fwd}

\begin{algorithm*}
\caption{Modified forward algorithm to compute the KL divergence loss for Soft-Label Chain CRFs}
\begin{algorithmic}
\Procedure{SoftLabelChainCrfLoss}{$\bm{q}, \eps(y^t, \bm{x}), \tau(y^t, y^{t-1}, \bm{x})$}
	\ForAll {label $y^0$}
		\State $\alpha^0_{y^0} \gets 0$
		\State $g^0_{y^0} \gets 0$
	\EndFor
	\For {$t = 1 \hdots T$}
		\ForAll {label $y^t$}
			\State $\alpha^t_{y^t} \gets \sum_{y^{t-1}}{\Big\{\alpha^{t-1}_{y^{t-1}} \exp{\big[\tau(y^t, y^{t-1}, \bm{x}) + \eps(y^t, \bm{x})\big]}\Big\}}$
			\State $g^t_{y^t} \gets \sum_{y^{t-1}}{\Big\{\big[g^{t-1}_{y^{t-1}} + \tau(y^t, y^{t-1}, \bm{x})\big] q^{t-1}_{y^{t-1}} + \big[\eps(y^t, \bm{x}) - \log{q^t_{y^t}}\big]\Big\}}$
		\EndFor
	\EndFor
	\State $Z \gets \sum_{y^T}{\alpha^T_{y^T}}$
	\State $G \gets \sum_{y^T}{g^T_{y^T} q^T_{y^T}}$
	\State $L \gets -G + \log{Z}$
	\State \Return $L$
\EndProcedure
\end{algorithmic}
\label{alg:fwd}
\end{algorithm*}

\begin{theorem}
    The modified forward algorithm computes the KL divergence loss for Soft-Label Chain CRFs. 
\end{theorem}
\begin{proof}
    It is known that the iterations on forward variables $\alpha$ computes partition function $Z(\bm{x})$. We only need to prove that
    \begin{align*}
        G
            &= \sum_t{\sum_{y^t, y^{t-1}}{q(y^t, y^{t-1} | \bm{x}) s(y^t, y^{t-1}, \bm{x})}} \nonumber \\
            &\qquad - \sum_t{\sum_{y^t}{q(y^t | \bm{x}) \log{q(y^t | \bm{x})}}}
    \end{align*}
    We first prove by induction that
    \begin{align*}
        g^t_{y^t}
            &= \sum_{t' < t}{\sum_{y^{t'}, y^{t'-1}}{q(y^{t'}, y^{t'-1} | \bm{x}) s(y^{t'}, y^{t'-1}, \bm{x})}} \nonumber \\
            &\qquad - \sum_{t' < t}{\sum_{y^{t'}}{q(y^{t'} | \bm{x}) \log{q(y^{t'} | \bm{x})}}} \nonumber \\
            &\qquad + \eps(y^t, \bm{x}) - \log{q^t_{y^t}}
    \end{align*}
    \textbf{Base Case:} When $t = 0$, 
    \begin{align*}
        g^0_{y^0}
            &= 0
    \end{align*}
    \textbf{Induction:} By inductive hypothesis, we have
    \begin{align*}
        g^{t-1}_{y^{t-1}}
            &= \sum_{t' < t-1}{\sum_{y^{t'}, y^{t'-1}}{q(y^{t'}, y^{t'-1} | \bm{x}) s(y^{t'}, y^{t'-1}, \bm{x})}} \nonumber \\
            &\qquad - \sum_{t' < t-1}{\sum_{y^{t'}}{q(y^{t'} | \bm{x}) \log{q(y^{t'} | \bm{x})}}} \nonumber \\
            &\qquad + \eps(y^{t-1}, \bm{x}) - \log{q^{t-1}_{y^{t-1}}}
    \end{align*}
    Therefore, 
    \begin{align*}
        g^t_{y^t}
            &= \sum_{y^{t-1}}{[g^{t-1}_{y^{t-1}} + \tau(y^t, y^{t-1}, \bm{x})] q^{t-1}_{y^{t-1}}} \nonumber \\
            &\qquad + \eps(y^t, \bm{x}) - \log{q^t_{y^t}} \\
            &= \sum_{t' < t-1}{\sum_{y^{t'}, y^{t'-1}}{q(y^{t'}, y^{t'-1} | \bm{x}) s(y^{t'}, y^{t'-1}, \bm{x})}} \nonumber \\
            &\qquad - \sum_{t' < t-1}{\sum_{y^{t'}}{q(y^{t'} | \bm{x}) \log{q(y^{t'} | \bm{x})}}} \nonumber \\
            &\qquad + \sum_{y^{t-1}}{[s(y^t, y^{t-1}, \bm{x}) - \log{q^{t-1}_{y^{t-1}}}] q^{t-1}_{y^{t-1}}} \nonumber \\
            &\qquad + \eps(y^t, \bm{x}) - \log{q^t_{y^t}} \\
            &= \sum_{t' < t}{\sum_{y^{t'}, y^{t'-1}}{q(y^{t'}, y^{t'-1} | \bm{x}) s(y^{t'}, y^{t'-1}, \bm{x})}} \nonumber \\
            &\qquad - \sum_{t' < t}{\sum_{y^{t'}}{q(y^{t'} | \bm{x}) \log{q(y^{t'} | \bm{x})}}} \nonumber \\
            &\qquad + \eps(y^t, \bm{x}) - \log{q^t_{y^t}}
    \end{align*}
    By following a similar derivation in the induction step, we have
    \begin{align*}
        G
            &= \sum_{y^{T}}{g^{T}_{y^{T}} q^{T}_{y^{T}}} \\
            &= \sum_{t' < T}{\sum_{y^{t'}, y^{t'-1}}{q(y^{t'}, y^{t'-1} | \bm{x}) s(y^{t'}, y^{t'-1}, \bm{x})}} \nonumber \\
            &\qquad - \sum_{t' < T}{\sum_{y^{t'}}{q(y^{t'} | \bm{x}) \log{q(y^{t'} | \bm{x})}}} \nonumber \\
            &\qquad + \sum_{y^{T}}{[s(y^T, y^{T-1}, \bm{x}) - \log{q^{T}_{y^{T}}}] q^{T}_{y^{T}}} \nonumber \\
            &= \sum_t{\sum_{y^t, y^{t-1}}{q(y^t, y^{t-1} | \bm{x}) s(y^t, y^{t-1}, \bm{x})}} \nonumber \\
            &\qquad - \sum_t{\sum_{y^t}{q(y^t | \bm{x}) \log{q(y^t | \bm{x})}}}
    \end{align*}
\end{proof}

\section{Equivalence of Forward-Backward and Backpropagation in Soft-Label Chain CRFs}
\label{sec:back}

\begin{lemma}
    Given forward and backward variables and partition function computed by the forward-backward algorithm
    \begin{align*}
        \alpha^t_{y^t}
            &= \sum_{y^{t-1}}{\alpha^{t-1}_{y^{t-1}} \exp{s(y^t, y^{t-1}, \bm{x})}} \\
        \beta^{t-1}_{y^{t-1}}
            &= \sum_{y^t}{\beta^t_{y^t} \exp{s(y^t, y^{t-1}, \bm{x})}} \\
        Z(\bm{x})
            &= \sum_{y^T}{\alpha^T_{y^T}}
    \end{align*}
    We have
    \begin{align*}
        \frac{\partial{Z(\bm{x})}}{\partial{\alpha^t_{y^t}}}
            &= \beta^t_{y^t}
    \end{align*}
\end{lemma}
\begin{proof}
    By induction on $t$. \\
    \textbf{Base case:} When $t = T$, 
    \begin{align*}
        \frac{\partial{Z(\bm{x})}}{\partial{\alpha^T_{y^T}}}
            &= 1 = \beta^T_{y^T}
    \end{align*}
    \textbf{Induction:} 
    \begin{align*}
        \frac{\partial{Z(\bm{x})}}{\partial{\alpha^{t-1}_{y^{t-1}}}}
            &= \sum_{y^t}{\frac{\partial{Z(\bm{x})}}{\partial{\alpha^t_{y^t}}} \frac{\partial{\alpha^t_{y^t}}}{\partial{\alpha^{t-1}_{y^{t-1}}}}} \\
            &= \sum_{y^t}{\beta^t_{y^t} \exp{s(y^t, y^{t-1}, \bm{x})}} \\
            &= \beta^{t-1}_{y^{t-1}}
    \end{align*}
\end{proof}

\begin{theorem}
    Backpropagation on the loss of Soft-Label Chain CRFs gives gradient
    \begin{align*}
        \frac{\partial{L}}{\partial{s(y^t, y^{t-1}, \bm{x})}}
            &= -q(y^t, y^{t-1} | \bm{x}) + p(y^t, y^{t-1} | \bm{x})
    \end{align*}
    where
    \begin{align*}
        p(y^t, y^{t-1} | \bm{x})
            &= \frac{1}{Z(\bm{x})} \alpha^{t-1}_{y^{t-1}} \beta^t_{y^t} \exp{s(y^t, y^{t-1}, \bm{x})}
    \end{align*}
\end{theorem}
\begin{proof}
    By induction on $t$. \\
    \textbf{Base case:} When $t = T$, 
    \begin{align*}
        &\frac{\partial{L}}{\partial{s(y^T, y^{T-1}, \bm{x})}} \\
            &= -q(y^T, y^{T-1} | \bm{x}) + \frac{1}{Z(\bm{x})} \frac{\partial{Z(\bm{x})}}{\partial{s(y^T, y^{T-1}, \bm{x})}} \\
            &= -q(y^T, y^{T-1} | \bm{x}) + \frac{\alpha^{T-1}_{y^{T-1}} \exp{s(y^T, y^{T-1}, \bm{x})}}{Z(\bm{x})} \\
            &= -q(y^T, y^{T-1} | \bm{x}) + p(y^T, y^{T-1} | \bm{x})
    \end{align*}
    \textbf{Induction:} 
    \begin{align*}
        &\frac{\partial{L}}{\partial{s(y^t, y^{t-1}, \bm{x})}} \\
            &= -q(y^t, y^{t-1} | \bm{x}) + \frac{1}{Z(\bm{x})} \frac{\partial{Z(\bm{x})}}{\partial{s(y^t, y^{t-1}, \bm{x})}} \\
            &= -q(y^t, y^{t-1} | \bm{x}) + \frac{1}{Z(\bm{x})} \frac{\partial{Z(\bm{x})}}{\partial{\alpha^t_{y^t}}} \frac{\partial{\alpha^t_{y^t}}}{\partial{s(y^t, y^{t-1}, \bm{x})}} \\
            &= -q(y^t, y^{t-1} | \bm{x}) + \frac{\beta^t_{y^t} \alpha^{t-1}_{y^{t-1}} \exp{s(y^t, y^{t-1}, \bm{x})}}{Z(\bm{x})} \\
            &= -q(y^t, y^{t-1} | \bm{x}) + p(y^t, y^{t-1} | \bm{x})
    \end{align*}
\end{proof}